\documentclass[conference]{IEEEtran}

\usepackage{amsmath,amssymb,amsfonts, amsthm}
\usepackage{url,hyperref}
\usepackage{algorithmic}
\usepackage{float}
\usepackage{xspace}
\usepackage{tcolorbox}
\usepackage{graphicx}
\usepackage{textcomp}
\usepackage{standalone}
\usepackage{xcolor}
\usepackage{adjustbox}

\newtheorem{theorem}{Theorem}
\newtheorem{lemma}{Lemma}

\title{Herglotz-NET: Implicit Neural Representation of Spherical~Data with Harmonic Positional Encoding}

\author{Th\'eo Hanon$^1$, Nicolas Mil-Homens Cavaco$^1$, John Kiely$^2$, Laurent Jacques$^1$ \\[2mm]
\small $^1$: INMA/ICTEAM, ULouvain, Belgium; $^2$: Yale University, CT, USA.
\thanks{E-mails: \protect\url{theo.hanon@student.uclouvain.be}, \protect\url{{nicolas.mil-homens,laurent.jacques}@uclouvain.be}, \protect\url{john.kiely@yale.edu}. Computational resources have been provided by the supercomputing facilities of UCLouvain (CISM) and the Consortium des Equipements de Calcul Intensif en Fédération Wallonie Bruxelles (CECI) funded by the Fond de la Recherche Scientifique de Belgique (FRS-FNRS). Part of this work is funded by the Belgian F.R.S.-FNRS (PDR project QuadSense; T.0160.24).}
}
\usepackage{commands}

\IEEEoverridecommandlockouts     
\begin{document}
\maketitle

\begin{abstract}
    Representing and processing data in spherical domains presents unique challenges, primarily due to the curvature of the domain, which complicates the application of classical Euclidean techniques. Implicit neural representations (INRs) have emerged as a promising alternative for high-fidelity data representation; however, to effectively handle spherical domains, these methods must be adapted to the inherent geometry of the sphere to maintain both accuracy and stability. In this context, we propose Herglotz-NET (HNET), a novel INR architecture that employs a harmonic positional encoding based on complex Herglotz mappings. This encoding yields a well-posed representation on the sphere with interpretable and robust spectral properties. Moreover, we present a unified expressivity analysis showing that any spherical-based INR satisfying a mild condition exhibits a predictable spectral expansion that scales with network depth. Our results establish HNET as a scalable and flexible framework for accurate modeling of spherical data.
\end{abstract}


\section{Introduction}
Over the past few years, implicit neural representations (INRs), such as SIREN~\cite{sitzmann2020implicitneuralrepresentationsperiodic} and NeRF \cite{mildenhall2020nerfrepresentingscenesneural}, have attracted significant attention as a powerful alternative to classical, explicitly discretized representations of continuous signals. 
In these models, neural networks learn the mapping from spatial coordinates to signal values, yielding differentiable and memory-efficient representations. 
These architectures have found applications in areas such as continuous 3D scene reconstruction~\cite{mildenhall2020nerfrepresentingscenesneural}, data compression~\cite{dupont2021coin}, as well as in solving partial differential equations and various inverse problems~\cite{sitzmann2020implicitneuralrepresentationsperiodic,liu2022recovery}. 
Despite their success in Euclidean settings, extending INRs to spherical domains poses additional challenges as they are prone to generate artifacts, especially near the poles.  
To enhance the capabilities of INRs in representing signals on the sphere, \cite{russwurm2024locationencoding}~has proposed to incorporate spherical harmonics (SH) into the positional encoding (PE) of INRs (cf.~Sec.~\ref{sec:related_work}). 
However, their approach, SPH-SIREN, requires closed-form evaluations (or recursive computation) of a potentially large list of spherical harmonics, which can be computationally intensive and prone to numerical instability at high orders.

In this work, we propose Herglotz-NET (HNET), a new implicit neural representation for spherical data, as an alternative to SPH-SIREN~\cite{russwurm2024locationencoding} that does not require explicit evaluations of spherical harmonics (cf.~Sec.~\ref{sec:inr_S2}). 
In particular, we propose to replace the Euclidean periodic functions used in INRs by the exponentiation of an Herglotz mapping~\cite[Sec. VII.7]{courant62}, a \emph{harmonic} PE.
In Sec.~\ref{sec:spectrum_expansion}, we generalize part of the spectral analysis of Euclidean INR  in~\cite{yuce2022structureddictionaryperspectiveimplicit} to Spherical INRs, providing insights on the expressive power of both HNET and SPH-SIREN~\cite{russwurm2024locationencoding} in function of their depth. 
Finally, in Sec.~\ref{sec:numerical_experiments}, we demonstrate in two experiments---a super-resolution application and the estimation of a continuous spherical Laplacian map---that, despite the simplicity of its PE, HNET's performance is on a par with that of SPH-SIREN, both schemes being superior to non-spherical SIREN for the Laplacian estimation. 

\section{Related Work}
\label{sec:related_work}
Typically, an implicit neural representation (INR) $f_\Theta: \mathcal{E} \to \bb K$ approximates a function \( f : \mathcal{E} \to \bb K \) (with $\bb K = \bb R$ or $\bb C$, and \(\mathcal{E} \subset \mathbb{R}^d\)) from its sampling $\cl D = \{(\bs x_i, y_i = f(\bs x_i))\}_{i=1}^N$ on a set of $N$ points $\{\bs x_i\}_{i=1}^N \subset \cl E$. This is done by \emph{training} the set of parameters $\Theta$ on $\cl D$ so that $f_\Theta$ interpolates $f$, \ie $f_\Theta(\bs x_i) \approx y_i$. An INR combines a positional encoding (PE) composed of $n$ neurons
\begin{equation}
\label{eq:PE-Euclidean-INR}
\psi_{\sin}(\bs x) = \Big(\sin(\bs \omega_i^\top \bs x + b_i)\Big)_{i=1}^n,    
\end{equation}
which extracts $n$ Fourier features from the input coordinates \(\bs{x} \in \mathcal{E}\), with a multilayer perceptron (MLP) whose activations are typically ReLU or sine wave.  Frequencies $\bs \omega_i$ and bias $b_i$ in the PE can be either considered as trainable weights 
or generated randomly, leading then to the so-called Random Fourier Features~\cite{Rahimi2007RandomFF}.
This periodic positional encoding layer of the INR network is key for accurate representations of a signal~\cite{tancik_fourier_2020} as it enables the representation of high-frequency signal features. 
The final architecture of an INR reads
\begin{equation}
\label{eq:INR}
f_{\Theta}(\bs x) := \bs{W}^{(Q)} (\phi^{(Q-1)}\circ \cdots \circ \phi^{(1)})\left(\psi(\bs x)\right),
\end{equation}
with the PE $\psi = \psi_{\sin}$, $\phi^{(q)}(\cdot) := \sigma^{(q)} (\bs{W}^{(q)} \cdot + \bs{b}^{(q)})$, $q=1,\ldots,Q-1$, and where each layer \( q \) is characterized by a weight matrix \( \bs{W}^{(q)} \in \bb{R}^{d_{q} \times d_{q-1}} \), a bias vector \( \bs{b}^{(q)} \in \bb{R}^{d_q} \), and an activation function \( \sigma^{(q)}: \bb{R} \rightarrow \bb{R} \) applied element-wise. 

Moreover, these architectures benefit from a comprehensive understanding of their underlying spectral properties~\cite{yuce2022structureddictionaryperspectiveimplicit}, namely that the frequency support of the initial positional encoding entirely governs the expressive capacity of the model. 
In addition, assuming the activation function admits a Taylor series expansion, 
each layer in the network expands the frequency spectrum by the order of the activation function’s expansion. For networks employing sinusoidal activation functions (\eg SIRENs \cite{sitzmann2020implicitneuralrepresentationsperiodic}) this expansion is theoretically infinite due to the infinite Taylor series, while in practice most of the coefficients of the polynomial expansion decay at a factorial rate so that some terms become negligible up to some order.

As they do not consider the spherical geometry, conventional INRs, such as SIRENs, can generate singular spherical functions prone to sampling artifacts, particularly at the poles. In \cite{russwurm2024locationencoding}, the authors addressed these limitations by introducing SPH-SIREN; a SIREN network integrating spherical harmonics (SH) $Y_{\ell m}(\theta,\varphi)$ with order $\ell \geq 0$ and moment $|m|\leq \ell$ \cite{courant62} into its PE, \ie they set the PE $\psi = \psi_{\rm SH}$ in~\eqref{eq:INR} with
\begin{equation*}
\bs x(\theta, \varphi)\in \bb S^2 \mapsto \psi_{\rm SH}(\bs x) = \{ {Y}_{\ell m}(\theta, \varphi) : |m|\leq \ell\}_{\ell=0}^{L_0}.
\end{equation*}
Our investigation revealed that, for orders above $L_0>25$, accurate closed-form expressions of SH in SPH-SIREN become impractical, while possible recursive estimation methods require high-precision arithmetic, limiting scalability.
\section{An alternative spherical INR on \(\bb S^2\)}
\label{sec:inr_S2}

We propose an alternative spherical INR relying on the Herglotz formalism~\cite[Sec. VII.7]{courant62}. Specifically, given a complex vector $\bs a \in \bb C^3$ satisfying the Herglotz condition $\bs a^\top \bs a =0$ (\eg the vector $(1,\im,0)$) and any twice complex-differentiable function $f: \bb C \to \bb R$, the function $g(\bs x) := f(\bs a^\top \bs x)$, composing $f$ with the Herlgotz mapping $\bs x \mapsto \bs a^\top\bs x$, is a \emph{harmonic} function, \ie its Laplacian $\Delta$ vanishes everywhere, $\Delta g(\bs x) = 0$. 

Interestingly, when restricted on the 2-sphere $\bb S^2$, the function 
$\bs x \to \exp(\bs a^\top\bs x)$
is a generating function of all spherical harmonics \cite{courant62}. In fact, given some $\omega_0 > 0$ and defining $\bs a_{\Re}, \bs a_{\Im} \in \bb R^3$ the real and imaginary parts of $\bs a$, on $\bb S^2$, the function 
\begin{equation}
\label{eq:Herglotz_mapping}
g(\bs x) = \exp\big(\omega_0 (\bs a^\top\bs x)\big),\ \bs x \in \bb S^2,
\end{equation}
is centered on $\bs a_{\Re}/\|\bs a_{\Re}\|$ with oscillations locally pointing in the direction of $\bs a_{\Im}$, and a frequency increasing with $\omega_0$ (see Fig.~\ref{fig:atoms}). This can be seen by first noting that $\bs a_{\Re}^\top\bs a_{\Im} = 0$ from the Herlgotz condition. Moreover, since $\bs a^\top\bs x = \bs a_{\Re}^\top \bs x + \im \bs a_{\Im}^\top\bs x$, $g(\bs x) = |g(\bs x)| \exp( \im \omega_0 (\bs a_{\Im}^\top\bs x))$, with a maximal amplitude $|g(\bs x)|=\exp( \omega_0 (\bs a_{\Re}^\top \bs x))$ on $\bs x \propto \bs a_{\Re}$. 

\begin{figure}[!t]
    \centering
    \includegraphics[width=\linewidth]{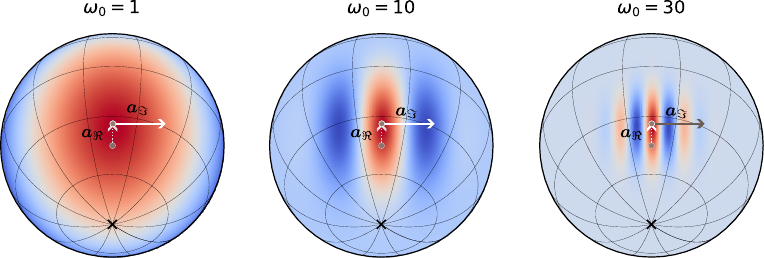}
    \vspace{-1.5em}
    \caption{Representation of the real part of the Herglotz atom $g(\bs x)$ for $\bs a = (1,1,1)/\sqrt{3} + \im (1,-1,0)/\sqrt{2}$, and different values of $\omega_0$. The function $g$ is centered on $\bs a_{\Re} = (1,1,1)/\sqrt{3}$ with oscillations locally oriented along the direction $\bs a_{\Im} = (1,-1,0)/\sqrt{2}$ and frequency proportional to $\omega_0$.}
    \vspace{-1em}
    \label{fig:atoms}
\end{figure}
\begin{figure}[!ht]
    \centering
    \includegraphics[width=.9\linewidth]{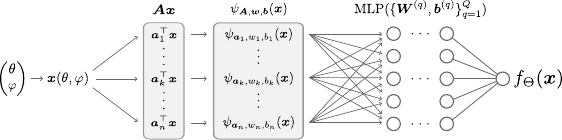}
    \caption{The Herglotz-NET architecture using the $n$-neuron PE defined by \eqref{eq:PE-hnet}.}
    \vspace{-1em}
    \label{fig:herglotz_layer}
\end{figure}

\noindent\textbf{Herglotz-NET}: From the above analysis, we propose an alternative INR for spherical data that preserves SIREN's flexibility and scalability while providing a well-defined framework on $\bb S^2$. Specifically, we define a Herglotz-NET (or HNET) by adapting the INR positional encoding $\psi$ in \eqref{eq:INR} using the mapping \eqref{eq:Herglotz_mapping}. Given $n$ Herglotz vectors $(\bs a_1, \ldots, \bs a_n)^\top =: \bs A \in \bb C^{n\times 3}$, with $\bs a_i^\top\bs a_i=0$ for all $i \in [n]$, $n$ trainable weights $(w_1, \ldots, w_n) =: \bs w$ and biases $(b_1, \ldots, b_n) =: \bs b$, the PE is composed of $n$ neurons 
\begin{equation}
\label{eq:PE-hnet}
\ts \psi_{\bs A, \bs w, \bs b}(\bs x) := \big\{\psi_{\bs a_i, w_i, b_i}(\bs x) = e^{\omega_0 (w_i \bs a_i^\top \bs x + b_i)}\big\}_{i=1}^n    
\end{equation}
encoding the location $\bs x(\theta,\varphi) \in \bb S^2$, with $(\theta,\varphi) \in [0,\pi] \times [0, 2\pi)$.   
Consequently, each neuron corresponds to a harmonic function if $\bs x \in \bb R^3$ if not restricted to $\bb S^2$. As for SIREN in~\eqref{eq:INR}, we add scaling factor $\omega_0 \in \bb R$ into the PE. This can drastically affect training---yielding faster convergence and enhanced stability when properly tuned. The PE is subsequently fed to an MLP with periodic activations $\sigma^{(q)}(\cdot) = \sin(\cdot)$ (see Fig.~\ref{fig:herglotz_layer}). 

\noindent\textbf{HNET initialization}: We initialize all the weights and biases of HNET using the scheme proposed in SIREN \cite{sitzmann2020implicitneuralrepresentationsperiodic}. Regarding the PE, we select $n$ \iid \emph{unit} random Herglotz vectors $\bs a_i \sim_{\iid} \bs a$, with $2\bs a_{\Re}$ picked uniformly at random on $\bb S^2$ (\ie $\|\bs a_{\Re}\|=1/2$) and $2\bs a_{\Im}$ picked uniformly at random on the circle $\bb S^1$ in the plane orthogonal to $2\bs a_{\Re}$ (with also $\|\bs a_{\Im}\|=1/2$). As one can show that the distribution of $\bs a^\top \bs x \in \bb C$ is unchanged if $\bs x \to \bs R \bs x$, with $\bs R \in SO(3)$ an arbitrary rotation, no neurons $e^{\omega_0 (w_i \bs a_i^\top \bs x)}$ favors, in expectation, a specific point of $\bb S^2$ or a specific direction of oscillation. 

\noindent\textbf{Spectral analysis}: One can easily understand the spectral properties of the function generated by each neuron of the PE, \ie its spectrum in the spherical harmonics basis $Y_{\ell m}(\theta,\varphi)$. Indeed, neglecting the biases, a Taylor series expansion gives  
\begin{equation}
    \label{eq:hk-SH-expansion}
\ts \exp\bigl(\omega_0 w_i (\bs a_i^\top \bs x)\bigr) = \sum_{\ell\ge 0} \frac{\omega_0^\ell}{\ell!} w_i^\ell (\bs a_i^\top \bs x)^\ell,
\end{equation}
where each term $(\bs a_i^\top \bs x)^\ell$ activates a unique spectral component at order $\ell$, as stated in Lemma \ref{lem:monomial_herglotz} below.
\begin{lemma}
\label{lem:monomial_herglotz}
Given a vector \(\bs a\in\mathbb{C}^3\), with \(\bs a^\top \bs a=0\) and \(\ell \in \bb N\) and $\|\bs a\|_2 = 1$, the function $h_\ell(\bs x)=(\bs a^\top \bs x)^\ell$ is harmonic in \(\bb R^3\) and its restriction to all $\bs x = \bs x(\theta,\varphi) \in \bb S^2$ reads
\begin{equation*}
\ts h_\ell(\bs x)=\sum_{|m|\leq \ell} c_{\ell m}(\bs a)\, Y_{\ell m}(\theta,\varphi),    
\end{equation*}
with $c_{\ell m}(\bs a)\in \bb C$ such that $\sum_{|m|\leq \ell} |c_{\ell m}|^2 \leq 4\pi (2\ell+1)$.
\end{lemma}

\begin{proof}
Given $\bs x = r \bs u \in \bb R^3$ with $r= \|\bs x\|$ and $\bs u \in \bb S^2$,  
since \(h_\ell\) is homogeneous of degree \(\ell\) (\ie \(h_\ell(\bs x)= r^\ell h_\ell(\bs u)\)), its Laplacian in $\bb R^3$ reads
\[
\ts \Delta h_\ell(\bs x)=r^{\ell-2} \big(\,\ell(\ell+1)h_\ell(\bs u)+\Delta_{\bb S^2}h_\ell(\bs u)\,\big).
\]
By definition of $\bs a$, \(h_\ell\) is harmonic (\ie \(\Delta h_\ell(\bs x)=0\)), and $\Delta_{\bb S^2}h_\ell(\bs u)=-\ell(\ell+1)h_\ell(\bs u)$, and \(h_\ell(\bs u)\) can thus be expanded in the basis \(\{Y_{\ell m}\}_{m=-\ell}^{\ell}\) of spherical harmonics of degree \(\ell\). Moreover, since $\|\bs a\| = 1$, $|h_\ell(\bs x)| \leq 1$ for $\bs x \in \bb S^2$, and $\sum_{|m|\leq \ell} |c_{\ell m}|^2 \leq 4\pi (2\ell+1)$ using Parseval's theorem~\cite{courant62}.
\end{proof}

From this lemma, we conclude that each neuron 
$\psi_{\bs a_i, w_i, b_i}(\bs x)$
in \eqref{eq:PE-hnet}, has bounded SH coefficients, \ie  
$|\hat{\psi}_{\ell m}| \leq \frac{\omega_0^\ell}{\ell!} w_i^{\ell} |c_{\ell m}| \|\bs a_i\|^\ell$, with $\hat{\psi}_{\ell m} := \scp{\psi_{\bs a_i, w_i, b_i}}{Y_{\ell m}}$. Since $\ell! \geq e(\ell/e)^\ell$, the spectrum of each neuron, computed with $S_\psi(\ell):=(2\ell +1)^{-1}\sum_{|m|\leq \ell} |\hat{\psi}_{\ell m}|^2$, is bounded by $S_\psi(\ell) \leq \bar S_\psi(\ell) := \frac{4\pi}{e^2} (\frac{e \omega_0 w_i \|\bs a_i\|}{\ell})^{2\ell}$, with $\bar S_\psi(\ell)$ reaching a maximum around $\ell \approx e \omega_0 w_i \|\bs a_i\|$, before to quickly decay to 0 when $\ell$ increases. 

This shows that each neuron spectrum is essentially bandlimited and characterized by a dominant order around $\ell \propto \omega_0 w_i \|\bs a_i\|$, which is consistent with the frequency of the oscillatory pattern in Fig.~\ref{fig:atoms}. The fixed $\omega_0\in\mathbb{R}$ and the trainable $w_i\in\mathbb{C}$ then allow the model to adapt itself to various spectrum configurations.

\begin{figure}[!t]
    \centering
    \includegraphics[width=\linewidth, trim={0 0cm 15.45cm 0}, clip]{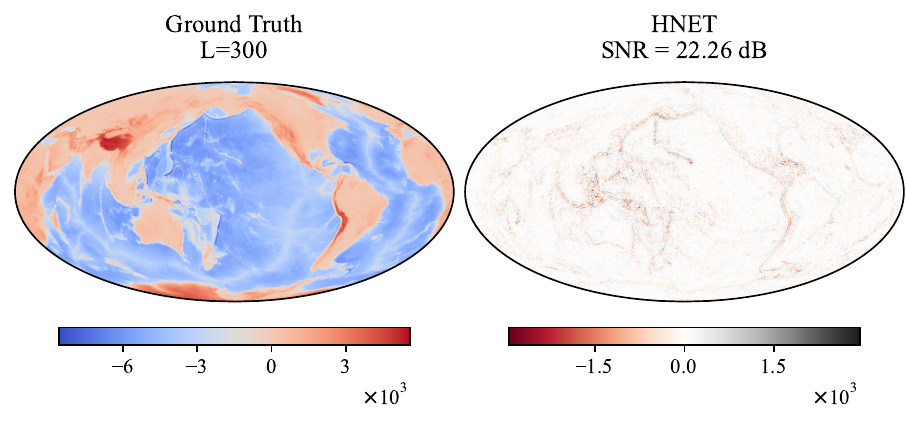}
    \vspace{-2.5em}
    \caption{(left) Earth representation on the fine grid ($L = 300$). (right)~Difference between this representation and HNET upsampled on the same grid.}
    \vspace{-1em}
    \label{fig:experience_high_resolution}
\end{figure}
 
\begin{figure*}[htbp] 
    \centering
    \includegraphics[width=.9\linewidth]{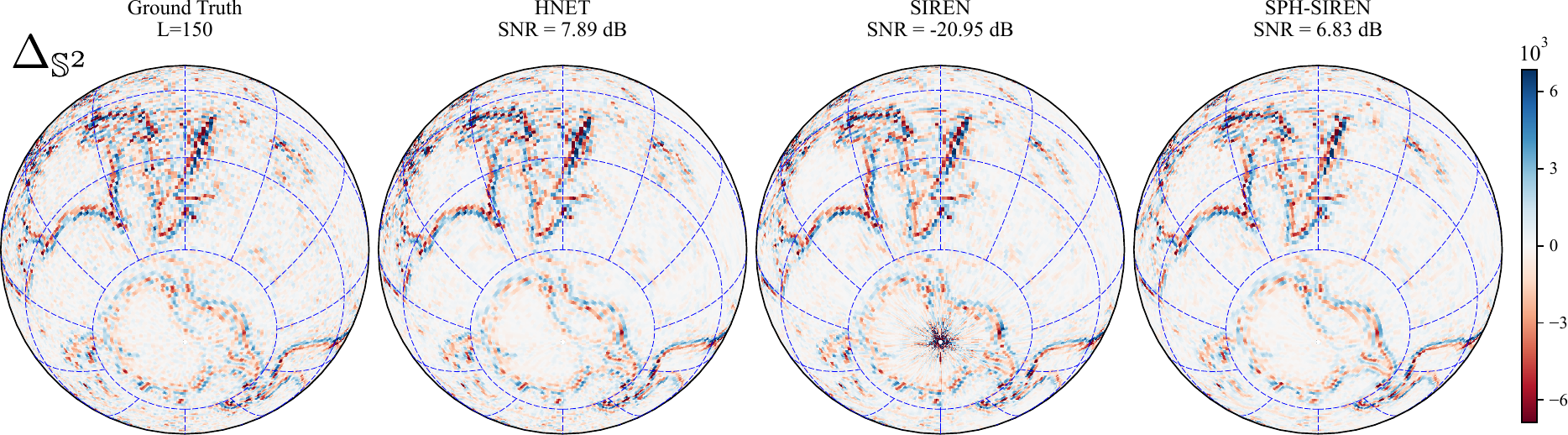}
    \caption{Comparison of the Laplacian of the Earth representation computed with spherical harmonics up to $L = 150$ (left) and the Laplacian of the models HNET, SIREN, and SPH-SIREN (from left to right). The ground truth Laplacian is computed using the spherical harmonic transform, while the model Laplacians are obtained via automatic differentiation of the continuous functions.}
    \vspace{-1em}
    \label{fig:experience_laplacian}
\end{figure*}

\section{Spectrum Expansion of Spherical INRs}
\label{sec:spectrum_expansion}
In this section, we build on the work of Yüce et al. \cite{yuce2022structureddictionaryperspectiveimplicit} by examining the spectral expansion of spherical-based INRs. Our goal is to shed light on the empirical performance of both HNET and SPH-SIREN~\cite{russwurm2024locationencoding} and to offer insights on the underlying properties of these models.

Our development relies on the key assumption that the HNET exponential functions in \eqref{eq:PE-hnet} and the subsequent MLP activations in both HNET and SPH-SIREN can be accurately approximated by truncated Taylor series. In particular, we approximate the exponential as $\exp(z) \approx \sum_{\ell=0}^{L_0} \beta_\ell z^\ell$, for a certain degree $L_0$, and represent the MLP activation functions as $\sigma(z) \approx \sum_{k=1}^{K} \alpha_k z^k$.
This assumption is motivated by the observation that many popular activation functions have rapidly decaying polynomial expansion coefficients—a property evident in $\exp(z)$ and, by extension, $\sin(z)$. 

This assumption allows us to reformulate the output of the first layer following the PE $\psi(\bs x)$ as a finite spectral expansion. In particular, for the \(j\)-th neuron we obtain
\begin{equation}
    \label{eq:poly_encoding}
\ts [\bs W^{(1)} \psi(\bs x) + \bs b^{(1)}]_j = \sum_{\ell=0}^{L_0}\sum_{|m| \leq \ell} c_{j\ell m}\, Y_{\ell m}(\theta,\varphi)
\end{equation}
for some $c_{j\ell m}\in \bb C$. Notably, Eq.~\eqref{eq:poly_encoding} is exact for SPH-SIREN, whereas it is only approximately fulfilled for HNET, given Lemma~\ref{lem:monomial_herglotz}. The following theorem formalizes the expansion property and delineates the expressivity limits of these architectures.
\begin{theorem}
    \label{theorem:limited_representation}
    Let $f_{\bs{\Theta}} : \bb S^2 \rightarrow \bb C$ be an INR  defined as in \eqref{eq:INR} with PE $\psi(\bs x)$ satisfying \eqref{eq:poly_encoding}, $Q \geq 1$ layers and polynomial activation functions $\sigma^{(q)}(z) = \sum_{k=0}^K \alpha_k z^k$. Then, there exist coefficients $\kappa_{\ell m} \in \bb C$ such that    
    \begin{equation}
        \ts f_{\bs{\Theta}}(\theta,\varphi) = \sum_{\ell=0}^{K^{Q-1} L_0}\sum_{|m| \leq \ell} \kappa_{\ell m} Y_{\ell m}(\theta,\varphi)
    \end{equation}
\end{theorem}
\begin{proof}
    See Appendix.
\end{proof}
With Theorem \ref{theorem:limited_representation} in mind, we see that both HNET and SIREN rely on the frequency content of their positional encodings to determine their expressive power. In HNET, the learnable parameters $w_i\in \bb C$ enables the network to adaptively select from a wide range of frequencies, which makes the model capable of representing an unbounded spectrum. 

\section{Numerical experiments}
\label{sec:numerical_experiments}
In this section, we present a comprehensive evaluation of our proposed HNET and compare its performance with two baselines: a conventional SIREN~\cite{sitzmann2020implicitneuralrepresentationsperiodic} and its spherical variant SPH-SIREN~\cite{russwurm2024locationencoding}. Our experiments use Earth data from EARTH2014 representation available in PySHTools datasets.

\noindent \textbf{Data \& Preprocessing}:
The training data consists on the evaluation of the EARTH2014 spherical harmonic representation on a Gauss-Legendre (GL) grid~\cite{price:s2fft} with $L = 150$. For testing, a finer GL grid with $L=300$ is used. Prior to model training, the training data is standardized. 

\noindent \textbf{Model Architectures}:  
We compare three models. First, HNET with a PE based on 50 Herglotz mappings, followed by 3 hidden layers with 100 features each, and uses a scaling parameter \(\omega_0=10\) (\#param = 30601). Secondly, SIREN with a positional encoding with 100 neurons, 3 hidden layers with 100 features each, and a first layer scaling parameter \(\omega_0=10\) (\#param = 30700). Lastly, SPH-SIREN with a positional encoding containing the  spherical harmonics up to \(L=10\) (resulting in 110 neurons), 2 hidden layers with 100 features each and we incorporate a modified activation for the first hidden layer given by \(\sin(\omega_0 x)\) with \(\omega_0=3\) (\#param = 22401).

\noindent \textbf{Training}: Each model was trained for 2000 epochs with a batch size of 2048, using the Adam optimizer~\cite{Kingma2014AdamAM} with a fixed learning rate of $10^{-3}$.

\subsection{Super-Resolution}
This experiment is conducted in a super-resolution context. After training the model on a coarse Gauss-Legendre grid~\cite{price:s2fft} (with $L=150$), we evaluate its generalization capability on a finer grid ($L = 300$). Fig.~\ref{fig:experience_high_resolution} illustrates the residuals between the HNET output and the ground truth on the $L=300$ grid. Quantitavely, HNET, SIREN, SPH-SIREN achieve SNR values of 22.26 dB, 22.25dB and 22.29dB, respectively. Visual inspection reveals that the residual patterns of all three models are nearly identical. Overall, all the three architectures perform equivalently on this task.

\subsection{Laplacian Reconstruction}
In this experiment, we evaluate the ability of the models to accurately reconstruct the spherical Laplacian on the training grid (\ie $L = 150$); for consistency, we employ the same trained models from the previous experiment.
 
Notably, while the super-resolution results indicated that all models perform equivalently, the Laplacian reconstruction reveals distinct differences. As shown in Fig.~\ref{fig:experience_laplacian}, SIREN—although visually accurate in reconstructing the Laplacian away from the poles—suffers from significant artifacts near the poles that hinder its SNR. In contrast, both HNET and SPH-SIREN exhibit no visible artifacts at the poles and faithfully reconstruct the Laplacian.  

In summary,  although SIREN remains an effective fitting tool on spherical domains, it loses the well-posed differentiability that is crucial for applications such as PDE solving and inverse problems. In contrast, both HNET and SPH-SIREN maintain well-posed differentiability, making them better suited for such tasks. Our results demonstrate that HNET effectively combines SIREN's fitting speed and capacity with the well-posedness inherent to SPH-SIREN on spherical domain.
\section{Conclusion}
In this paper, we developed a novel implicit neural representation, HNET, adapted to  spherical data. By leveraging the Herglotz formalism to construct a harmonic positional encoding, HNET circumvents the need for explicit spherical harmonic evaluations. 
Our experiments have shown that HNET provides performances that are on par with the SPH-SIREN architecture that explicitly use spherical harmonics. Moreover, we demonstrated that, assuming a polynomial approximation of the activation functions, the spectral expressivity of both HNET and SPH-SIREN scales with the network depth, enabling the model to capture a rich (possibly) infinite set of frequencies on the sphere. 

\section*{Appendix : Proof of Theorem 1}
\noindent
The proof of Theorem \ref{theorem:limited_representation} requires the following lemma. We also use the handy notation $\sum_{(\ell,m)}^{[\![L']\!]} \equiv \sum_{\ell=0}^{L'}\sum_{|m|\leq \ell}$ (for $L' \in \bb N$).  
\begin{lemma}
    \label{lem:1}
    Given a complex value polynomial $p_K(z) = \sum_{k=1}^{K} a_k z^k$ of degree $K\geq 0$ and an order $L$, for any spherical function $f(\theta,\varphi) := \sum_{\ell=0}^L\sum_{|m|\leq \ell} c_{\ell m} Y_{\ell m}(\theta, \varphi)$ (with $c_{\ell m} \in \bb C$) with bandlimit $L$, $p_K \circ f$ has a bandlimit $KL$, \ie there exist coefficients $\kappa_{\ell m}\in \bb C$ such that 
    $$
        \ts p_K\big(f(\theta,\varphi)\big) = \sum_{\ell=0}^{KL}\sum_{|m|\leq \ell} \kappa_{\ell m} Y_{\ell m}(\theta, \varphi). 
    $$
\end{lemma}
\vspace{-4mm}
\begin{proof}
   By induction, we start by observing that $K=0$ is trivial since $Y_{0,0}(\theta,\varphi)$ is constant. Next, for general $K>0$, the proof only needs to consider  monomials of degree ${K+1}$ since $p_{K+1}(z) = a_{K+1} z^{K+1} + p_K(z)$. We can develop  
   \begin{align*}
   &\ts f(\theta, \varphi)^{K+1}\!= \sum_{(\ell,m)}^{[\![KL]\!]} \sum_{(\ell',m')}^{[\![L]\!]} \kappa_{\ell m} c_{\ell' m'} Y_{\ell m}(\theta,\varphi) Y_{\ell' m'}(\theta,\varphi)\\
   &\ts = \sum_{(\ell,m)}^{[\![KL]\!]} \sum_{(\ell',m')}^{[\![L]\!]}\kappa_{\ell m} c_{\ell' m'} \sum_{(k,p)}^{[\![\ell + \ell']\!]} \alpha_{kp} Y_{k p}(\theta, \varphi) \\
   &\ts = \sum_{(\ell, m)}^{[\![(K+1)L]\!]} {\hat \kappa}_{\ell m}  Y_{\ell m}(\theta, \varphi), 
    \end{align*}
    where the second line uses the contraction rule of SHs \cite[Sec.~5.8]{devanathan_angular_2002}, the coefficients $\alpha_{kp}\in \bb C$ involve Wigner 3-$j$ symbols, and $\hat \kappa_{\ell m} \in \bb C$.
\end{proof}
To prove Theorem~\ref{theorem:limited_representation}, we first consider a bias-free scenario ($\mathbf{b}^{(q)} = \mathbf{0}$ for all $q = 1, \dots, Q-1$). 
The $i$-th component of the pre-activation input of the second layer reads
   \begin{align*}
       &\ts {z}_i^{(2)} = 
        \sum_{j=1}^{d^{(2)}} W^{(2)}_{ij} \sigma^{(1)} \big([\bs W^{(1)} \psi(\bs x) + \bs b^{(1)}]_j\big) \\
       & \ts = \sum_{j=1}^{d^{(2)}} W^{(2)}_{ij} \sigma^{(1)}\big( \sum_{(\ell,m)}^{[\![L_0]\!]} c_{j \ell m} Y_{\ell m}(\theta,\varphi)\big) \\
       &\ts = \sum_{(\ell,m)}^{[\![KL_0]\!]} \big(\sum_{j=1}^{d^{(2)}} W^{(2)}_{ij}\kappa_{j\ell m}\big) Y_{\ell m}(\theta, \varphi)\\
       &\ts = \sum_{(\ell,m)}^{[\![KL_0]\!]} \hat \kappa_{i\ell m} Y_{\ell m}(\theta, \varphi), 
   \end{align*}
with $\hat \kappa_{i\ell m}\in \bb C$.
The second line comes from \eqref{eq:poly_encoding}, and the third one from Lemma~\ref{lem:1} for some $\kappa_{j\ell m} \in \bb C$.
The result of Theorem~\ref{theorem:limited_representation} follows by applying Lemma~\ref{lem:1} for each layer, that is, $(Q-1)$ times.
Adding biases shifts the constant term of the polynomial activation function. Consequently, biases do not alter the above reasoning, which concludes the proof.


\begin{thebibliography}{10}

\bibitem{sitzmann2020implicitneuralrepresentationsperiodic}
V.~Sitzmann, J.~N.~P. Martel, A.~W. Bergman, D.~B. Lindell, and G.~Wetzstein,
  ``Implicit neural representations with periodic activation functions,'' in
  \emph{Proceedings of the 34th International Conference on Neural Information
  Processing Systems}, ser. NIPS '20.\hskip 1em plus 0.5em minus 0.4em\relax
  Red Hook, NY, USA: Curran Associates Inc., 2020.

\bibitem{mildenhall2020nerfrepresentingscenesneural}
B.~Mildenhall, P.~P. Srinivasan, M.~Tancik, J.~T. Barron, R.~Ramamoorthi, and
  R.~Ng, ``Nerf: representing scenes as neural radiance fields for view
  synthesis,'' \emph{Commun. ACM}, vol.~65, no.~1, p. 99–106, Dec. 2021.

\bibitem{dupont2021coin}
E.~Dupont, A.~Golinski, M.~Alizadeh, Y.~W. Teh, and A.~Doucet, ``Coin:
  Compression with implicit neural representations,'' in \emph{Neural
  Compression: From Information Theory to Applications--Workshop@ ICLR 2021},
  2021.

\bibitem{liu2022recovery}
R.~Liu, Y.~Sun, J.~Zhu, L.~Tian, and U.~S. Kamilov, ``Recovery of continuous 3d
  refractive index maps from discrete intensity-only measurements using neural
  fields,'' \emph{Nature Machine Intelligence}, vol.~4, no.~9, pp. 781--791,
  2022.

\bibitem{russwurm2024locationencoding}
M.~Rußwurm, K.~Klemmer, E.~Rolf, R.~Zbinden, and D.~Tuia, ``Geographic
  location encoding with spherical harmonics and sinusoidal representation
  networks,'' in \emph{Proceedings of the International Conference on Learning
  Representations (ICLR)}, 2024.

\bibitem{courant62}
R.~Courant and D.~Hilbert, \emph{Methods of Mathematical Physics}.\hskip 1em
  plus 0.5em minus 0.4em\relax New York: Wiley, 1962, vol.~1.

\bibitem{yuce2022structureddictionaryperspectiveimplicit}
G.~Yüce, G.~Ortiz-Jiménez, B.~Besbinar, and P.~Frossard, ``A structured
  dictionary perspective on implicit neural representations,'' \emph{arXiv
  preprint arXiv:2112.01917}, 2022.

\bibitem{Rahimi2007RandomFF}
A.~Rahimi and B.~Recht, ``Random features for large-scale kernel machines,'' in
  \emph{Neural Information Processing Systems}, 2007. 
\bibitem{tancik_fourier_2020}
M.~Tancik, P.~Srinivasan, B.~Mildenhall, S.~Fridovich-Keil, N.~Raghavan,
  U.~Singhal, R.~Ramamoorthi, J.~Barron, and R.~Ng, ``Fourier features let
  networks learn high frequency functions in low dimensional domains,''
  \emph{Advances in neural information processing systems}, vol.~33, pp.
  7537--7547, 2020. 
\bibitem{price:s2fft}
M.~A. Price and J.~D. McEwen, ``Differentiable and accelerated spherical
  harmonic and wigner transforms,'' \emph{Journal of Computational Physics},
  vol. 510, p. 113109, 2024.

\bibitem{Kingma2014AdamAM}
D.~P. Kingma and J.~Ba, ``Adam: A method for stochastic optimization,''
  \emph{CoRR}, vol. abs/1412.6980, 2014.
\bibitem{devanathan_angular_2002}
V.~Devanathan, \emph{Angular Momentum Techniques in Quantum Mechanics}, ser.
  Fundamental Theories of Physics.\hskip 1em plus 0.5em minus 0.4em\relax
  Kluwer Academic Publishers, 2002, no. 108, iSBN:978-0-306-47123-0.

\end{thebibliography}
\end{document}